\documentclass[11pt]{article}
\usepackage[utf8]{inputenc}
\usepackage{amsmath}
\usepackage{amssymb}
\usepackage{amsthm}
\usepackage{mathtools}
\usepackage{xcolor}

\usepackage{multirow}
\usepackage{tabularx}
\usepackage{enumitem}
\usepackage{booktabs}
\usepackage{caption, subcaption}

\usepackage[colorlinks]{hyperref}
\usepackage{authblk}

\title{A Simple Proof of the Mixing of Metropolis-Adjusted Langevin Algorithm under Smoothness and Isoperimetry}
\author[1]{Yuansi Chen}
\author[2]{Khashayar Gatmiry}
\affil[1]{Duke University}
\affil[2]{MIT}
\date{}

\newcommand{\target}{\mu}

\newcommand{\isop}{\psi}
\newcommand{\dist}{\mathfrak{D}}

\newcommand{\hq}{\hat{q}}
\newcommand{\hp}{\hat{p}}

\newcommand{\vq}{\mathbf{q}}

\newcommand{\proposal}{\mathcal{P}}
\newcommand{\transition}{\mathcal{T}}
\newcommand{\step}{\eta}
\newcommand{\tole}{\epsilon}

\newcommand{\tmix}{\tau_{\text{mix}}}
\newcommand{\tvdis}{{\rm{d}}_{{\rm{TV}}}}

\newcommand{\gradf}{\nabla f}
\newcommand{\hessf}{\nabla^2 f}

\newcommand{\Accept}{\mathcal{A}}

\newcommand{\J}{\mathbf{J}}
\newcommand{\D}{\mathbf{D}}

\newcommand{\Lparam}{L}

\newcommand{\Cl}{\Upsilon_\ell}

\newcommand{\Clh}{\Upsilon_{\frac{\ell}{2}}}

\newcommand{\hF}{\hat{F}}

\newcommand{\trH}{\Upsilon}

\newcommand{\warm}{M}

\newcommand{\defn}{:=}
\newcommand{\rdefn}{=:}

\theoremstyle{plain}

\newtheorem{theorem}{Theorem}

\newtheorem{lemma}{Lemma}

\newlength{\widebarargwidth}
\newlength{\widebarargheight}
\newlength{\widebarargdepth}

\makeatletter
\long\def\@makecaption#1#2{
        \vskip 0.8ex
        \setbox\@tempboxa\hbox{\small {\bf #1:} #2}
        \parindent 1.5em  
        \dimen0=\hsize
        \advance\dimen0 by -3em
        \ifdim \wd\@tempboxa >\dimen0
                \hbox to \hsize{
                        \parindent 0em
                        \hfil
                        \parbox{\dimen0}{\def\baselinestretch{0.96}\small
                                {\bf #1.} #2
                                }
                        \hfil}
        \else \hbox to \hsize{\hfil \box\@tempboxa \hfil}
        \fi
        }
\makeatother

\long\def\comment#1{}
\definecolor{battleshipgrey}{rgb}{0.52, 0.52, 0.51}
\definecolor{darkgray}{rgb}{0.66, 0.66, 0.66}
\definecolor{darkgreen}{rgb}{0.0, 0.2, 0.13}
\definecolor{darkspringgreen}{rgb}{0.09, 0.45, 0.27}
\definecolor{dukeblue}{rgb}{0.0, 0.0, 0.61}
\definecolor{olivedrab7}{rgb}{0.24, 0.2, 0.12}
\definecolor{darkblue}{rgb}{0.0, 0.0, 0.55}
\definecolor{darkscarlet}{rgb}{0.34, 0.01, 0.1}
\definecolor{candyapplered}{rgb}{1.0, 0.03, 0.0}
\definecolor{ao(english)}{rgb}{0.0, 0.5, 0.0}
\definecolor{applegreen}{rgb}{0.55, 0.71, 0.0}

\makeatletter
\def\moverlay{\mathpalette\mov@rlay}
\def\mov@rlay#1#2{\leavevmode\vtop{%
   \baselineskip\z@skip \lineskiplimit-\maxdimen
   \ialign{\hfil$\m@th#1##$\hfil\cr#2\crcr}}}
\newcommand{\charfusion}[3][\mathord]{
    #1{\ifx#1\mathop\vphantom{#2}\fi
        \mathpalette\mov@rlay{#2\cr#3}
      }
    \ifx#1\mathop\expandafter\displaylimits\fi}
\makeatother

\DeclareMathOperator{\diag}{diag}

\DeclareMathOperator{\trace}{trace}

\newcommand{\dims}{\ensuremath{d}}
\newcommand{\real}{\ensuremath{\mathbb{R}}}

\newcommand{\naturalnum}{\ensuremath{\mathbb{N}}}

\newcommand{\Ind}{\ensuremath{\mathbb{I}}}
\newcommand{\borel}{\ensuremath{\mathcal{B}}}

\newcommand{\Exs}{\ensuremath{{\mathbb{E}}}}
\newcommand{\Prob}{\ensuremath{{\mathbb{P}}}}

\newcommand{\Normal}{\ensuremath{\mathcal{N}}}

\DeclarePairedDelimiterX{\infdivx}[2]{(}{)}{%
  #1\;\delimsize\|\;#2%
}

\newcommand{\brackets}[1]{\left[ #1 \right]}
\newcommand{\parenth}[1]{\left( #1 \right)}

\newcommand{\braces}[1]{\left\{ #1 \right \}}
\newcommand{\abss}[1]{\left| #1 \right |}

\newcommand{\ceils}[1]{\left\lceil #1 \right \rceil}

\newcommand{\tp}{^\top}

\newcommand{\bmat}[1]{\begin{bmatrix} #1 \end{bmatrix}}

\newcommand{\matsnorm}[2]{\left|\!\left|\!\left| #1 \right|\!\right|\!\right|_{{#2}}}
\newcommand{\vecnorm}[2]{\left\| #1\right\|_{#2}}
\newcommand{\enorm}[1]{\vecnorm{#1}{2}}

\begin{document}

\maketitle

\begin{abstract}
  We study the mixing time of Metropolis-Adjusted Langevin algorithm (MALA) for sampling a target density on $\mathbb{R}^d$. We assume that the target density satisfies $\psi_\mu$-isoperimetry and that the operator norm and trace of its Hessian are bounded by $L$ and $\Upsilon$ respectively. Our main result establishes that, from a warm start, to achieve $\epsilon$-total variation distance to the target density, MALA mixes in $O\left(\frac{(L\Upsilon)^{\frac12}}{\psi_\mu^2} \log\left(\frac{1}{\epsilon}\right)\right)$ iterations. Notably, this result holds beyond the log-concave sampling setting and the mixing time depends on only $\Upsilon$ rather than its upper bound $L d$.  In the $m$-strongly logconcave and $L$-log-smooth sampling setting, our bound recovers the previous minimax mixing bound of MALA~\cite{wu2021minimax}.
\end{abstract}

\section{Introduction}
The goal of this paper is to clarity what current proof techniques can show about the mixing time of Metropolis-Adjusted Langevin algorithm (MALA) for sampling a smooth target density on $\real^\dims$ that satisfies an isoperimetry. Although the mixing time of MALA is well studied in the case of log-concave distributions, there is a question regarding how the mixing properties change when a slightly weaker assumption, such as isoperimetry, is made. In the study of Unadjusted Langevin Algorithm (ULA), Vempala and Wibisono~\cite{vempala2019rapid} questioned the impact of the isoperimetry assumption on mixing. We would like to ask the same question for MALA. In particular, we believe that existing proof techniques such as truncated conductance and multivariate integration by parts that appeared in~\cite{wu2021minimax} are sufficient to answer this question. 

Freund et al.~\cite{freund2022convergence} investigated when the mixing time of Langevin algorithms becomes independent of the dimension. They show that the mixing time of such algorithms can be made to depend on the trace of the Hessian of the target density, rather than the dimension itself. Consequently, when combined with composite sampling, this leads to an algorithm with dimension-independent mixing time. We aim to explore whether the mixing time of MALA can similarly be shown to depend solely on the trace of the Hessian, instead of the $\dims^{\frac12}$ dependency in~\cite{wu2021minimax}.

Based on the two motivations above, we investigate the mixing time of MALA to sample a target density $\target$ which satisfies the Cheeger's isoperimetric inequality with coefficient $\isop_\target$ and has the operator norm and the trace of its Hessian bounded by $\Lparam$ and $\trH$ respectively. We present a simple proof that shows from a warm start, MALA mixes in $O\parenth{\frac{(\Lparam\trH)^{\frac12}}{\isop_\target^2} \log\parenth{\frac{1}{\tole}}}$ iterations to achieve $\tole$-total variation distance to the target density. Compared to the existing work~\cite{wu2021minimax}, our result demonstrates that the existing proof techniques presented there such as truncated conductance and multivariate integration by parts are sufficient to deal with sampling on a smooth target density with isoperimetry. When applied to sampling from an $m$-strongly logconcave and $L$-log-smooth density, our mixing time upper bound recovers the previous minimax mixing bound of MALA~\cite{wu2021minimax}.  Additionally, the dependency on $\trH$ in our mixing time bound allows us to get weaker dimension dependency in several scenarios. 

\section{Preliminaries}
\label{sec:preliminaries}
In this section, we introduce the MALA algorithm and provide necessary background knowledge to establish mixing time of Markov chains. 

\subsection{Markov chain basics}
We consider the problem of sampling from a target measure $\mu$ with density with respect to the Lebesgue measure on $\real^\dims$. Given a \textit{Markov chain with transition kernel} $P: \real^\dims \times \borel(\real^n) \to \real_{\geq 0}$ where $\borel(\real^\dims)$ denotes the Borel $\sigma$-algebra on $\real^\dims$, the $k$-step transition kernel $P^k$ is defined recursively by $P^k(x, dy) \defn \int_{z \in \real^n} P^{k-1}(x, dz) P(z, dy)$. 
The Markov chain is called \textit{reversible} with respect to the target measure $\mu$ if 
\begin{align*}
  \mu(dx) P(x, dy) =  \mu(dy) P(y, dx).
\end{align*}
One can associate the Markov chain with a transition operator $\transition_P$.
\begin{align*}
  \transition_P(\nu) (S) \defn \int_{y \in \real^n } d\nu(y) P(y, S), \quad \forall S\in \borel(\real^n).
\end{align*}
In words, when $\nu$ is the distribution of the current state, $\transition_P(\nu)$ gives the distribution of the next state. And $\transition_P^n(\nu) \defn \transition_{P^n}(\nu)$ is the distribution of the state after $n$ steps. 

\paragraph{$s$-conductance.} For $s \in (0, 1)$, we define the $s$-conductance $\Phi_s$ of the Markov chain $P$ with its stationary measure $\mu$ as follows
\begin{align}
  \label{eq:def_s_conductance}
  \Phi_s (P) \defn \inf_{S: s < \mu(S) < 1-s} \frac{\int_S P(x, S^c) \mu(dx)}{\min\braces{\mu(S), \mu(S^c)} - s}.
\end{align}
When compared to conductance (the case $s=0$), $s$-conductance allows us to ignore small parts of the distribution where the conductance is difficult to bound. 

\paragraph{Lazy chain.} Given a Markov chain with transition kernel $P$. We define its lazy variant $P^{\text{lazy}}$, which stays in the same state with probability at least $\frac{1}{2}$, as
\begin{align*}
  P^{\text{lazy}}(x, S) \defn \frac{1}{2} \delta_{x \to S} + \frac{1}{2} P(x,S).
\end{align*}
Here $\delta_{x \to \cdot}$ is the Dirac measure at $x$. 
Since the lazy variant only slows down the convergence rate by a constant factor, we study lazy Markov chains in this paper to simplify our
theoretical analysis.

\paragraph{Total-variance distance.} Let the total variation (TV) distance between two probability distributions $\mathcal{P}_1, \mathcal{P}_2$ be 
\begin{align*}
  \tvdis(\mathcal{P}_1, \mathcal{P}_2) \defn \sup_{A \in \borel(\real^n)} \abss{\mathcal{P}_1(A) - \mathcal{P}_2(A)},
\end{align*}
where $\borel(\real^n)$ is the Borel sigma-algebra on $\real^n$. If $\mathcal{P}_1$ and $\mathcal{P}_2$ admit densities $p_1$ and $p_2$ respectively, we may write $\tvdis(\mathcal{P}_1, \mathcal{P}_2) = \frac{1}{2} \int \abss{p_1(x) - p_2(x)} dx$.

\paragraph{Warm start.} We say an initial measure $\mu_{0}$ is \textit{$\warm$-warm} if it satisfies
\begin{align*}
  \sup_{S \in \borel(\real^n)} \frac{\mu_{0}(S)}{\mu(S)} \leq \warm. 
\end{align*}

\paragraph{Mixing time.} For an error tolerance $\epsilon \in (0, 1)$, the total variance distance $\epsilon$-mixing time of the Markov chain $P$ with initial distribution $\mu_{0}$ and target distribution $\mu$ is defined as
\begin{align*}
  \tmix^P(\epsilon, \mu_{0}, \mu) \defn \inf\braces{n \in \naturalnum \mid \tvdis\parenth{\mathcal{T}^n_P (\mu_{0}), \mu} \leq \epsilon }. 
\end{align*}

\subsection{MALA basics}
Given a state $x_k$ at the $k$-th iteration, MALA proposes a new state $y_{k+1}\sim\Normal(x_k-h\gradf(x_k),2h\Ind_\dims)$. Then to obtain the next state, it decides to accept or reject $y_{k+1}$ using a Metropolis-Hastings correction. We use $\proposal_x=\Normal(x-h\gradf(x),2h\Ind_\dims)$ to denote the proposal kernel of MALA at $x$.

One step of MALA can be seen as one step of Metropolized Hamiltonian Monte Carlo (HMC) with a single leapfrog step~\cite{neal2011mcmc}. Using the notation from HMC, MALA works as follows: At each iteration, given the step-size $\step = (2h)^{\frac12}$ and the current state $q_0 \in \real^d$, draw an independent Gaussian $p_0 \sim \Normal(0, \Ind_d)$, the proposal for the next state is $q_\step$ defined as
\begin{align}
  \label{eq:leapfrog_HMC_recursion_single}
  q_{\step} &= q_{0} + \step p_0 - \frac{\step^2}{2} \gradf(q_0) \notag \\
  p_{\step} &= p_0 - \frac{\step}{2} \gradf(q_0) - \frac{\step}{2} \gradf(q_\step).
\end{align}
A Metropolis-Hastings correction is used to ensure the correct stationary measure, with acceptance probability
\begin{align}
  \label{eq:HMC_acceptance_single}
  \Accept(q_0, p_0) &\defn \min\braces{1, \frac{\exp\parenth{-f(q_\step) - \frac{1}{2} \enorm{p_\step}^2}}{\exp\parenth{-f(q_0) - \frac{1}{2} \vecnorm{p_0}{2}^2} }}.
\end{align}
See~\cite{neal2011mcmc,wu2021minimax} for more details about the connection between MALA and HMC. In the following, we adopt the Metropolized HMC with a single leapfrog step notation with $\step = (2h)^{\frac12}$.

\subsection{Assumptions on the target density}
Here are the regularity assumptions that one could impose on the target density. 
\begin{itemize}
  \item A twice differentiable function $f: \real^\dims \to \real$ is \textit{$\Lparam$-smooth} if $\hessf_x \preceq \Lparam \Ind_\dims, \forall x \in \real^\dims$. A target density $\target \propto e^{-f}$ is called \textit{$\Lparam$-log-smooth} is $f$ is $\Lparam$-smooth. 
  \item A target density $\target \sim e^{-f}$ satisfies the \textit{Cheeger's isoperimetric inequality} with coefficient $\isop_\target$ if for any partition $(S_1, S_2, S_3)$ of $\real^\dims$, the following inequality is satisfied
  \begin{align*}
    \target(S_3)\geq \isop_\target \cdot \dist(S_1,S_2)\cdot \target(S_1)\target(S_2).
  \end{align*}
  where the distance between sets $S_1$, $S_2$ is defined as $\dist(S_1,S_2)=\inf_{x\in S_1,y\in S_2}\enorm{x-y}$.
  \item A target density $\mu \propto e^{-f}$ is \textit{subexponential-tailed} if there exists $\lambda > 0$ such that
  \begin{align*}
    \lim_{\enorm{x} \to \infty} e^{-\lambda \enorm{x}} e^{-f(x)} \to 0.
  \end{align*}
  Note that if $\mu$ satisfies the Cheeger isoperimetric inequality, then it is also subexponential-tailed. More precisely, the Poincar\'e constant of $\mu$ follows from Cheeger's isoperimetric coefficient by Cheeger's inequality~\cite{cheeger2015lower,maz1960classes} (see also Theorem 1.1 in~\cite{milman2009role}), and the Poincar\'e constant of $\mu$ implies subexponential tail bounds following Gromov and Milman~\cite{gromov1983topological} (see also Theorem 1.7 of~\cite{gozlan2015dimension}). 
\end{itemize}

\section{Main results}

\begin{theorem}
  \label{thm:main_MALA}
  Let $\target \propto e^{-f}$ be a $\Lparam$-log-smooth target density on $\real^\dims$, with a Cheeger's isoperimetric coefficient $\isop_\target$. Let $\trH = \sup_{x \in \real^\dims} \trace(\hessf_x)$. Then there exists universal constants $c_0$ and $c_1$ such that for any error tolerance $\tole\in(0,1)$, from any $\warm$-warm initial measure $\mu_0$, the $\tole$-mixing time of lazy MALA with step-size 
  \begin{align*}
    h =\frac{\step^2}{2}= \frac{c_0}{\max\braces{(\Lparam \trH)^{\frac12}, \Lparam \log \parenth{\frac{(\Lparam \trH)^{\frac14}}{\psi_\target} \frac{M}{\tole}}} }
  \end{align*}
  satisfies
  \begin{align*}
    \tmix(\tole, \mu_0, \target) \leq c_1 \frac{\max\braces{(\Lparam \trH)^{\frac12}, \Lparam \log \parenth{\frac{(\Lparam \trH)^{\frac14}}{\psi_\target} \frac{M}{\tole}}}}{\psi_\target^2} \log\parenth{\frac{M}{\tole}}.
  \end{align*}
\end{theorem}
The theorem suggests that taking a step-size of order $\frac{1}{(\Lparam \trH)^{\frac12}}$ results in a mixing time of order $\frac{(\Lparam \trH)^{\frac12}}{\psi_\target^2}$, after ignoring the logarithmic factors.
Note that under the $\Lparam$-log-smooth assumption, the supremum over all traces of Hessian $\trH$ is always upper bounded by $\Lparam \dims$. In particular, in the setting of sampling $m$-strongly logconcave and $\Lparam$-log-smooth density in~\cite{wu2021minimax}, our mixing time upper bound is $\dims^{\frac12} \kappa$ where $\kappa = \frac{\Lparam}{m}$ because the isoperimetric coefficient for $m$-strongly logconcave density is $m^{\frac12}$. Our mixing time upper bound recovers the previous minimax mixing bound of MALA~\cite{wu2021minimax} up to constant factors. Furthermore, here is a log-concave sampling example where our bound is better: let the Hessian satisfy
\begin{align*}
  \diag\parenth{\frac{\Lparam}{\dims}, \frac{\Lparam}{2\dims}, \ldots, \frac{\Lparam}{2\dims}} \preceq \hessf_x \preceq \diag\parenth{\Lparam, \frac{\Lparam}{\dims}, \ldots, \frac{\Lparam}{\dims}}.
\end{align*}
A naive application of~\cite{wu2021minimax} gives a dimension dependency of $\dims^{\frac32}$ for the mixing time because $\kappa$ is of order $\dims$ and it enforces a smaller step size. However, the bound in Theorem~\ref{thm:main_MALA} gives a mixing time of order $\dims$ because $\trH$ is only of order $2\Lparam$. 

Remark that in the log-concave sampling setting, a warm start stated in Theorem~\ref{thm:main_MALA} can be found algorithmically via running underdamped Langevin algorithms for several steps~\cite{altschuler2023faster}.

To prove Theorem~\ref{thm:main_MALA}, we follow the framework of lower bounding the conductance of Markov chains to analyze mixing times~\cite{sinclair1989approximate,lovasz1993random}. The following lemma reduces the problem of mixing time analysis to lower bounding the $s$-conductance $\Phi_s$. 

\begin{lemma}[Lovasz and Simonovits~\cite{lovasz1993random}]
  \label{lem:lovasz_lemma}
  Consider a reversible lazy Markov chain with kernel $P$ and stationary measure $\target$. Let $\mu_{0}$ be an $M$-warm initial measure. Let $0 < s < \frac{1}{2}$. Then 
  \begin{align*}
    \tvdis\parenth{\transition_P^n (\mu_0), \target} \leq Ms + M \parenth{1 - \frac{\Phi_s^2}{2}}^n .
  \end{align*}
\end{lemma}
Using $s$-conductance in the place of the usual conductance enables us to focus on a high probability region with a good mixing behavior. In order to lower bound the conductance in a high probability region, we follow the proof of Theorem~3 in~\cite{wu2021minimax}, which is an extension of a result by Lov\'asz~\cite{lovasz1999hit} (see also Lemma 2 in~\cite{dwivedi2018log}). Informally, it states that as long as $\tvdis\parenth{\transition_x, \transition_y} \leq 1-\rho$ whenever $\enorm{x-y} \leq \Delta$ in a high probability region, we have $\Phi_s \geq \frac{\rho\Delta}{\psi_\mu}$. Therefore, it remains to bound the transition overlap, namely the TV-distance between $\transition_x$ and $\transition_y$ for two close points $x, y$.

We bound the TV-distance via the following triangle inequality,
\begin{align*}
  \tvdis\parenth{\transition_x, \transition_y} \leq \tvdis\parenth{\transition_x, \proposal_x} + \tvdis\parenth{\proposal_x, \proposal_y} + \tvdis\parenth{\proposal_y, \transition_y}.
\end{align*}
To bound the above term, we proceed in two steps. First, we bound the proposal overlap, which is the TV-distance between $\proposal_x$ and $\proposal_y$ for two close points $x, y$. By Pinsker's inequality, this TV-distance is upper bounded by the KL-divergence. Since the two proposals distributions are Gaussian, the KL-divergence has a closed form.  Based on this idea, we have the following lemma 
\begin{lemma}
  \label{lem:proposal_overlap}
  Let $\step^2 \Lparam \leq 1$. For any $x,y\in \real^\dims$, we have
  \begin{align*}
    \tvdis\parenth{\proposal_x, \proposal_y} \leq \frac{2}{\step} \enorm{x- y}.
  \end{align*}
\end{lemma}
The proof is omitted, as it follows directly from Lemma~7 in~\cite{dwivedi2018log}. Next, we bound the acceptance rate. Since $\tvdis\parenth{\transition_x, \proposal_x} = 1 - \Exs_{z \sim \Normal(0, \Ind_{\dims})}\brackets{\Accept(x, z)}$. It is sufficient to control the acceptance rate $\Accept$ in Eq.~\eqref{eq:HMC_acceptance_single} with high probability.
\begin{lemma}
  \label{lem:acceptance_lower_bound}
  Assume $f$ is $\Lparam$-smooth and the step-size satisfies
  \begin{align*}
    \step^4 \leq \frac{1}{4096\max\braces{\Lparam^2 \log(1/\delta), \Lparam \trH}},
  \end{align*}
  For any $\delta \in (0, 1)$, there exists a set $\Lambda \subset \real^{\dims \times \dims}$ with $\Prob_{(q_0, p_0) \sim \mu}((q_0, p_0) \in \Lambda) \geq 1-\delta$, such that for $(q_0, p_0) \in \Lambda$ and $(q_\step, p_\step)$ defined in Eq~\eqref{eq:leapfrog_HMC_recursion_single}, we have
  \begin{align}
    \label{eq:acceptance_lower_bound_in_lemma}
    - f(q_\step) - \frac{1}{2} \enorm{p_\step}^2 + f(q_0) + \frac{1}{2} \enorm{p_0}^2 \geq  - \frac{1}{4}.
  \end{align}
\end{lemma}
The proof of this lemma constitute the novelty of this paper, and it is provided in Section~\ref{sec:acceptance_rate_control}.

\begin{proof}[Proof of Theorem~\ref{thm:main_MALA}]
  The proof of Theorem~\ref{thm:main_MALA} consists of adapting the proof of Theorem~3 in~\cite{wu2021minimax} to incorporate the novel acceptance rate control in Lemma~\ref{eq:acceptance_lower_bound_in_lemma}. The rest of the proof is exactly the same. We provide it for completeness. 

  Applying Lemma~\ref{lem:lovasz_lemma}, with the choice
  \begin{align*}
    s = \frac{\tole}{2M}, n \geq \frac{2}{\Phi_s^2} \log \frac{2M}{\tole},
  \end{align*}
  it is sufficient to lower bound the conductance $\Phi_s$. Let $\delta \in (0, 1)$ to be chosen later, and $\Lambda$ be the set in Lemma~\ref{lem:acceptance_lower_bound}. Define 
  \begin{align*}
    \Xi \defn \braces{q_0 \in \real^\dims: \Prob_{p_0 \sim \Normal(0, \Ind_\dims)} \parenth{(q_0, p_0)\in \Lambda} \geq \frac{15}{16} }. 
  \end{align*}
  Then $\Prob_{q_0 \sim \target}(q_0 \in \Xi) \geq 1 - 16 \delta$. Applying Lemma~\ref{lem:acceptance_lower_bound}, we obtain
  \begin{align*}
    \Accept(q_0, p_0) \geq e^{-\frac{1}{4}}, \text{ for } (q_0, p_0) \in \Lambda.
  \end{align*}
  Consequently, we have
  \begin{align*}
    \tvdis\parenth{\transition_{q_0}, \proposal_{q_0}} = 1 - \Exs_{p_0 \sim \Normal(0, \Ind_\dims)}[\Accept(q_0, p_0)] \leq 1 - \frac{15}{16} e^{-\frac14} \leq \frac{1}{3}, \forall q_0 \in \Xi. 
  \end{align*}
  Together with Lemma~\ref{lem:proposal_overlap}, we have
  \begin{align}
    \label{eq:tv_dist_lazyxy}
    \tvdis\parenth{\transition_x^{\text{lazy}}, \transition_y^{\text{lazy}}} &\leq \frac{1}{2} + \frac{1}{2} \tvdis\parenth{\transition_x, \transition_y} \notag \\
    &\leq \frac{1}{2} + \frac{1}{2} \parenth{\frac{2}{3} + \frac{2\enorm{x-y}}{\step}},
  \end{align}
  where $\transition_x^{\text{lazy}}$ is the lazy version of $\transition_x$. To lower bound the $s$-conductance, consider any $S$ measurable set with $\target(S) \in (s, \frac12]$. Define the sets 
  \begin{align*}
    S_1 \defn \braces{x \in S \mid \transition_x^{\text{lazy}} < \frac{1}{16}}, S_2 \defn \braces{x \in S^c \mid \transition_x^{\text{lazy}} < \frac{1}{16}}, S_3 \defn (S_1 \cup S_2)^c. 
  \end{align*}
  Note that $S_1$ and $S_2$ are the sets that do no conduct well. Now there are two cases.
  \paragraph{If $\target(S_1) \leq \frac12 \target(S)$ or $\target(S_2) \leq \frac12 \target(S^c)$,}then the sets that conduct well are large enough:
  \begin{align*}
    \int_S \transition^{\text{lazy}}_x (S^c) \target(dx) \geq \frac{1}{32} \target(S).
  \end{align*}
  \paragraph{Otherwise,} for any $x \in S_1 \cap \Xi$ and $y \in S_2 \cap \Xi$, we have
  \begin{align*}
    \tvdis\parenth{\transition_x^{\text{lazy}}, \transition_y^{\text{lazy}}} \geq \abss{\transition_x^{\text{lazy}}(S) - \transition_y^{\text{lazy}}(S)} \geq \frac{15}{16} - \frac{1}{16} \geq \frac{7}{8}.
  \end{align*}
  Together with Eq.~\eqref{eq:tv_dist_lazyxy}, we obtain $\dist(S_1 \cap \Xi,S_2 \cap \Xi) \geq \frac{\step}{24}$. The isoperimetric inequality on $\target$ leads to 
  \begin{align*}
    \target(S_3 \cup \Xi^c) \geq \frac{\step}{24} \isop_\target \target(S_1 \cap \Xi) \target(S_2 \cap \Xi). 
  \end{align*}
  Since $\target(\Xi) \geq 1 - 16 \delta$, after cutting off the parts outside $\Xi$ under the condition 
  \begin{align}
    \label{eq:delta_condition}
    16\delta\leq \min\braces{\frac{1}{4}s, \frac{1}{1536} \step \isop_\target s},
  \end{align}
  we obtain that
  \begin{align*}
    \target(S_3) \geq c \step \psi_\target \target(S),
  \end{align*}
  for a universal constant $c > 0$. In this case, we have
  \begin{align*}
    \int_S \transition^{\text{lazy}}_x (S^c) \target(dx) \geq \frac{1}{32} \target(S_3) \geq \frac{c}{32} \step \psi_\target \target(S).
  \end{align*}
  Combining the two cases, the $s$-conductance is lower bounded as follows
  \begin{align*}
    \Phi_s \geq \frac{1}{32} \min\braces{1, c \step \isop_\target}. 
  \end{align*}
  $\delta$ has to be chosen according to Eq.~\eqref{eq:delta_condition} and the step-size $\step$ has to be chosen according to the first equation in Lemma~\ref{lem:acceptance_lower_bound}.
\end{proof}

\section{Acceptance rate control}
\label{sec:acceptance_rate_control}
In this section, we prove Lemma~\ref{lem:acceptance_lower_bound}. That is, we lower bound the MALA acceptance rate with step-size $\step$ and with high probability over $(q_0, p_0)$.

\subsection{Intermediate concentration results}
\label{sub:intermediate_concentration_results}
With $\trH$ and $\ell$, define the constant
\begin{align}
  \label{eq:def_Cl}
  \Cl \defn \trH + 2 (\ell-1) \Lparam.
\end{align}
\begin{lemma}
  \label{lem:grad_norm_bound}
  Let integer $\ell \geq 1$. Suppose $e^{-f}$ is subexponential-tailed, $f$ is $\Lparam$-smooth and $\sup_{x \in \real^\dims} \trace(\hessf_x) \leq \trH$, let $\Cl$ be defined in Eq.~\eqref{eq:def_Cl}, then
  \begin{align*}
    \brackets{\Exs_{q \sim e^{-f}} \enorm{\gradf(q)}^{2\ell}}^{\frac{1}{\ell}} \leq \Cl.
  \end{align*}
\end{lemma}

\begin{lemma}
  \label{lem:pHp_bound}
  Let integer $\ell \geq 1$ and $x\in \real^\dims$. Suppose $\trace(\hessf_x) \leq \trH$, let $\Cl$ be defined in Eq.~\eqref{eq:def_Cl}, then
  \begin{align*}
    \brackets{\Exs_{p \sim \Normal(0, \Ind_\dims)} \parenth{p \hessf_{x} p}^{\ell}}^{\frac{1}{\ell}} \leq \Cl
  \end{align*}
\end{lemma}

\begin{lemma}
  \label{lem:pHp_at_qt_bound}
  Let integer $\ell \geq 1$, $t \geq 0$ and $t^2 \Lparam \leq 1$. Suppose $f$ is $\Lparam$-smooth and the trace of its Hessian is upper-bounded $\sup_{x \in \real^\dims} \trace(\hessf_x) \leq \trH$, let $\Cl$ be defined in Eq.~\eqref{eq:def_Cl}, then
  \begin{align*}
    \brackets{\Exs_{(q_0, p_0) \sim \mu} \parenth{p_0 \tp \hessf_{q_t} p_0}^\ell }^{\frac{1}{\ell}} \leq 2 \Cl,
  \end{align*}
  where $q_t = q_0 + tp_0 - \frac{t^2}{2} \gradf(q_0)$ as defined in Eq.~\eqref{eq:leapfrog_HMC_recursion_single}.
\end{lemma}

\begin{lemma}
  \label{lem:grad_diff_norm_bound}
  Let integer $\ell \geq 1$, $0 \leq t \leq \step$ and $\step^2 \Lparam \leq 1$. Suppose $f$ is $\Lparam$-smooth and $\sup_{x \in \real^\dims} \trace(\hessf_x) \leq \trH$, let $\Cl$ be defined in Eq.~\eqref{eq:def_Cl}, then
  \begin{align*}
    \brackets{\Exs_{(q_0, p_0) \sim \mu} \enorm{\gradf(q_t) - \gradf(q_0)}^{2\ell} }^{\frac{1}{\ell}} &\leq 4 t^2 \Lparam \Cl, \text{ and } \\
    \brackets{\Exs_{(q_0, p_0) \sim \mu} \enorm{\gradf(q_t) - \gradf(q_\step)}^{2\ell} }^{\frac{1}{\ell}} &\leq 4 (\step-t)^2 \Lparam \Cl,
  \end{align*}
  where $q_t = q_0 + tp_0 - \frac{t^2}{2} \gradf(q_0)$ as defined in Eq.~\eqref{eq:leapfrog_HMC_recursion_single}.
\end{lemma}

\subsection{Proof of Lemma~\ref{lem:acceptance_lower_bound}}
Given an integer $\ell \geq 1$, we introduce the $\ell$-th norm of a random variable $\vq$ on $\real^\dims$ as 
\begin{align*}
  \matsnorm{\vq}{\ell} \defn \parenth{\Exs \enorm{\vq}^\ell}^{\frac{1}{\ell}}.
\end{align*}
This $\matsnorm{\cdot}{\ell}$ norm satisfies the triangle inequality
\begin{align*}
  \matsnorm{\vq_1 + \vq_2}{\ell} \leq \matsnorm{\vq_1}{\ell} + \matsnorm{\vq_2}{\ell}.
\end{align*}
\begin{proof}[Proof of Lemma~\ref{lem:acceptance_lower_bound}]
As it was observed in previous studies~\cite{dwivedi2018log,wu2021minimax}, naively applying Taylor series expansion is not sufficient to obtain good bounds of the acceptance rate. The novelty of our proof lies in the careful applications of integration by parts. 

Let $\Delta_\step(q_0, p_0) \defn f(q_\step) - f(q_0) + \frac12 \enorm{p_\step}^2 - \frac12\enorm{p_0}^2$. For simplicity, we omit the $(q_0, p_0)$ dependence if it is clear from the context. The goal is to provide a high probability upper bound of $\Delta_\step$.
Recall that the proposal of MALA with step-size $\step$ is 
\begin{align*}
  q_\step &= q_0 + \step p_0 - \frac{\step^2}{2} \gradf(q_0), \\
  p_\step &= p_0 - \frac{\step}{2} \gradf(q_0) - \frac{\step}{2} \gradf(q_\step).
\end{align*}
Expanding the terms in the acceptance rate~\eqref{eq:acceptance_lower_bound_in_lemma}, we have
\begin{align*}
  f(q_\step) - f(q_0) &= \int_0^\step \gradf(q_t) \tp (p_0 - t \gradf(q_0)) dt, \text{ and } \\
  \enorm{p_\step}^2 - \enorm{p_0}^2
  &= \enorm{p_0 - \step \gradf(q_0) + \frac{\step}{2}\parenth{\gradf(q_0) - \gradf(q_\step)} }^2 - \enorm{p_0}^2 \\
  &=  \enorm{p_0 - \step \gradf(q_0)}^2 + 2 \parenth{p_0 - \step \gradf(q_0)} \tp \parenth{ \frac{\step}{2}\parenth{\gradf(q_0) - \gradf(q_\step)} } \\
  &\quad + \enorm{ \frac{\step}{2}\parenth{\gradf(q_0) - \gradf(q_\step)} }^2 - \enorm{p_0}^2
\end{align*}
After rearranging the terms based on $\gradf(q_0)$ and $p_0$, we obtain
\begin{align}
  \label{eq:Ham_diff_init}
  &\quad \Delta_{\step}(q_0, p_0) \notag \\
  &=f(q_\step) - f(q_0) + \frac12 \enorm{p_\step}^2 - \frac12\enorm{p_0}^2 \notag \\
  &=\underbrace{\gradf(q_0) \tp \brackets{- \int_0^\step t (\gradf(q_t) - \gradf(q_\step)) dt} + p_0 \tp \brackets{\int_0^\step \parenth{\gradf(q_t) - \frac12 \gradf(q_0) - \frac12 \gradf(q_\step)} dt }}_{\rdefn B_\step(q_0, p_0)} \notag \\
  &\quad + \frac{\step^2}{8} \enorm{\gradf(q_\step) - \gradf(q_0)}^2.
\end{align}
For an even integer $\ell \geq 1$, we consider the $\ell$-th moment of $\Delta_\step$ as follows
\begin{align}
  \label{eq:Delta_bound_init}
  \brackets{\Exs_{(q_0, p_0) \sim \mu \times \Normal(0, \Ind_\dims)} \Delta_{\step}^{\ell}}^{\frac{1}{\ell}} \leq 2 \parenth{\brackets{\Exs_{(q_0, p_0) \sim \mu \times \Normal(0, \Ind_\dims)} B_\step^{\ell}}^{\frac{1}{\ell}} + \frac{\step^2}{8} \Exs \brackets{ \enorm{\gradf(q_\step) - \gradf(q_0)}^{2\ell}}^{\frac{1}{\ell}} }.
\end{align}
The second term above is bounded according to Lemma~\ref{lem:grad_diff_norm_bound}. It remains to bound $\Exs B_\step^\ell$. Define the vector $v_\step(q_0, p_0) \in \real^{2\dims}$ as
\begin{align}
  \label{eq:def_v}
  v_\step(q_0, p_0) \defn \bmat{ - \int_0^\step t (\gradf(q_t) - \gradf(q_\step)) dt \\ \int_0^\step \parenth{\gradf(q_t) - \frac12 \gradf(q_0) - \frac12 \gradf(q_\step)} dt  }.
\end{align}
Then we can write  
\begin{align*}
  B_\step(q_0, p_0) = v_\step(q_0, p_0)\tp \bmat{\gradf(q_0) \\ p_0}.
\end{align*}
The main strategy to bound $B_\step$ is to apply integration by parts jointly on $q_0$ and $p_0$. The derivative of $v_\step$ with respect to $(q_0, p_0)$ is 
\begin{align*}
  &\quad \D v_\step(q_0, p_0) \\
  &= \bmat{\frac{\partial v_1}{\partial q_0} & \frac{\partial v_1}{\partial p_0}\\ \frac{\partial v_2}{\partial q_0} & \frac{\partial v_2}{\partial p_0} } \\
  &= \bmat{
  -\int_0^\step t \brackets{\parenth{\Ind_\dims - \frac{t^2}{2} \hessf_{q_0}} \hessf_{q_t} - \parenth{\Ind_\dims - \frac{\step ^2}{2} \hessf_{q_0}} \hessf_{q_\step}} dt & \int_0^\step t \parenth{t\hessf_{q_t} dt - \step \hessf_{q_\step}} dt \\  
  \int_0^\step \parenth{\Ind_\dims - \frac{t^2}{2} \hessf_{q_0}} \hessf_{q_t} dt - \frac{\step}{2} \hessf_{q_0} - \frac{\step}{2} \parenth{\Ind_\dims - \frac{\step^2}{2} \hessf_{q_0}} \hessf_{q_\step} & \int_0^\step t \hessf_{q_t} dt - \frac{\step^2}{2} \hessf_{q_\step} }
\end{align*} 
We have
\begin{align}
  \label{eq:B_ell_power}
  \Exs B_\step^\ell &= \Exs B_\step^{\ell - 1} v_\step \tp \bmat{\gradf(q_0) \\ p_0} \notag \\
  &\overset{(i)}{=}  \Exs B_\step^{\ell - 1} \underbrace{\trace\parenth{\D v_\step}}_{A_1} \notag \\
  &\quad + (\ell-1) \Exs B_\step^{\ell - 2} \underbrace{v_\step \tp \bmat{\hessf_{q_0} & 0 \\ 0 & \Ind_\dims } v_\step}_{A_2} \notag \\
  &\quad + (\ell-1) \Exs B_\step^{\ell - 2}  \underbrace{\bmat{\gradf(q_0) \\ p_0} \tp \D v_\step v_\step}_{A_3} 
\end{align}
(i) applies multivariate integration by parts (we integrate $\bmat{\gradf(q_0) \\ p_0} e^{-f(q_0)- \frac12 \enorm{p_0}^2}$ and the boundary term vanishes because $e^{-f}$ is subexponential-tailed), which results in three terms. 
\paragraph{The first term in Eq.~\eqref{eq:B_ell_power}.}  We have
\begin{align*}
  A_1 &= \trace\parenth{\D v_\step} \\
  &= \trace\parenth{ \int_0^\step \frac{t^3}{2} \hessf_{q_0} \hessf_{q_t} - \frac{\step^4}{4} \hessf_{q_0} \hessf_{q_\step} } \\
  &\leq \frac12 \step^4\Lparam \trH. 
\end{align*}

\paragraph{The second term in Eq.~\eqref{eq:B_ell_power}.} We have
\begin{align*}
  &\quad A_2 \\
  &=v_\step \tp \bmat{\hessf_{q_0} & 0 \\ 0 & \Ind_\dims } v_\step \\
  &= \parenth{\int_0^\step t (\gradf(q_t) - \gradf(q_\step)) dt}\tp \hessf_{q_0} \parenth{\int_0^\step t (\gradf(q_t) - \gradf(q_\step)) dt}  \\
  &\quad + \enorm{\int_0^\step \gradf(q_t) dt - \frac{\step}{2} \gradf(q_\step) - \frac{\step}{2} \gradf(q_0)}^2 \\
  &\leq \Lparam \brackets{\int_0^\step t \enorm{\gradf(q_t) - \gradf(q_\step)} dt }^2 \\
  &\quad + \frac{1}{4} \brackets{ \int_0^\step \enorm{\gradf(q_t) - \gradf(q_0)} + \enorm{\gradf(q_t) - \gradf(q_\step)} dt }^2.
\end{align*}
Consequently, we have
\begin{align*}
  &\quad \matsnorm{A_2}{\frac{\ell}{2}} \\
  &\leq \Lparam \matsnorm{\int_0^\step t \enorm{\gradf(q_t) - \gradf(q_\step)} dt }{\ell}^{2} + \frac{1}{4} \matsnorm{\int_0^\step \enorm{\gradf(q_t) - \gradf(q_0)} + \enorm{\gradf(q_t) - \gradf(q_\step)} dt }{\ell}^{2}\\
  &\overset{(i)}{\leq} \Lparam\brackets{\frac{\step^2}{2} \sup_{t \in [0, \step]}\matsnorm{\gradf(q_t) - \gradf(q_\step)}{\ell} }^{2} + \frac{1}{4} \brackets{\step \sup_{t \in [0, \step]} \matsnorm{\gradf(q_t) - \gradf(q_\step)}{\ell} + \matsnorm{\gradf(q_t) - \gradf(q_0)}{\ell}}^{2} \\
  &\overset{(ii)}{\leq} \step^6 \Lparam^2 \Clh + 4 \step^4 \Lparam \Clh \\
  &\leq 5 \step^4 \Lparam \Clh.
\end{align*}
(i) follows from Jensen's inequality. (ii) follows from Lemma~\ref{lem:grad_diff_norm_bound}. 
\paragraph{The third term in Eq.~\eqref{eq:B_ell_power}.} $A_3$ consists of four following terms 
\begin{align*}
  &\quad A_3 \\
  &=\gradf(q_0) \tp \brackets{-\int_0^\step t \brackets{\parenth{\Ind_\dims - \frac{t^2}{2} \hessf_{q_0}} \hessf_{q_t} - \parenth{\Ind_\dims - \frac{\step ^2}{2} \hessf_{q_0}} \hessf_{q_\step}} dt } \brackets{- \int_0^\step t (\gradf(q_t) - \gradf(q_\step)) dt} \\
  &\quad + \gradf(q_0) \tp \brackets{\int_0^\step t \parenth{t\hessf_{q_t} dt - \step \hessf_{q_\step}} dt} \brackets{\int_0^\step \parenth{\gradf(q_t) - \frac12 \gradf(q_0) - \frac12 \gradf(q_\step)} dt} \\
  &\quad + p_0\tp \brackets{\int_0^\step \parenth{\Ind_\dims - \frac{t^2}{2} \hessf_{q_0}} \hessf_{q_t} dt - \frac{\step}{2} \parenth{\Ind_\dims - \frac{\step^2}{2} \hessf_{q_0}} \hessf_{q_\step} - \frac{\step}{2} \hessf_{q_0}} \brackets{- \int_0^\step t (\gradf(q_t) - \gradf(q_\step)) dt} \\
  &\quad + p_0\tp \brackets{\int_0^\step t \hessf_{q_t} dt - \frac{\step^2}{2} \hessf_{q_\step}} \brackets{\int_0^\step \parenth{\gradf(q_t) - \frac12 \gradf(q_0) - \frac12 \gradf(q_\step)} dt} \\
  &\leq \step^2 \Lparam \int_0^\step t \enorm{\gradf(q_0)} \enorm{\gradf(q_t) - \gradf(q_\step)} dt \\
  &\quad + \frac12 \step^3 \Lparam \int_0^\step \enorm{\gradf(q_0)} \parenth{\enorm{\gradf(q_t) - \gradf(q_0)} + \enorm{\gradf(q_t) - \gradf(q_\step)}} dt \\ 
  &\quad + \int_0^\step \Lparam^{\frac12} \parenth{\parenth{p_0\tp \hessf_{q_t} p_0}^{\frac12} + \frac{t^2}{2}\Lparam \parenth{p_0\tp \hessf_{q_0} p_0}^{\frac12} } dt \brackets{\int_0^\step t \enorm{\gradf(q_t) - \gradf(q_\step)} dt } \\
  &\quad + \frac12 \step^2 \Lparam \brackets{  \parenth{p_0\tp \hessf_{q_t} p_0}^{\frac12} + \parenth{p_0\tp \hessf_{q_\step} p_0}^{\frac12} } \brackets{\int_0^\step \parenth{\enorm{\gradf(q_t) - \gradf(q_0)} + \enorm{\gradf(q_t) - \gradf(q_\step)}} dt}.
\end{align*}
The $\frac{\ell}{2}$-th moment of the first term in the above equation can be bounded as follows
\begin{align*}
  &\quad \Exs \parenth{\step^2 \Lparam \int_0^\step t \enorm{\gradf(q_0)} \enorm{\gradf(q_t) - \gradf(q_\step)} dt}^{\frac{\ell}{2}} \\
  &\leq \step^{\frac{3\ell}{2} - 1} \Lparam^{\frac{\ell}{4}} \int_0^\step \Exs \parenth{t \Lparam^{\frac12} \enorm{\gradf(q_0)}\enorm{\gradf(q_t) - \gradf(q_\step)}}^{\frac{\ell}{2}} dt \\
  &\leq \frac{1}{2} \step^{\frac{3\ell}{2} - 1} \Lparam^{\frac{\ell}{4}} \int_0^\step (t^2\Lparam)^{\frac{\ell}{2}} \Exs \enorm{\gradf(q_0)}^{\ell} + \Exs\enorm{\gradf(q_t) - \gradf(q_\step)}^{\ell} dt \\
  &\overset{Lem.~\ref{lem:grad_norm_bound},~\ref{lem:grad_diff_norm_bound}}{\leq} \parenth{\step^{5} \Lparam^{\frac{3}{2}} \Clh }^{\frac{\ell}{2}}.
\end{align*}
Similarly, applying Lemma~\ref{lem:grad_norm_bound}~\ref{lem:pHp_bound}~\ref{lem:pHp_at_qt_bound} and~\ref{lem:grad_diff_norm_bound}, the $\frac{\ell}{2}$-th moments of the other three terms are bounded by $ \parenth{\step^{5} \Lparam^{\frac{3}{2}} \Clh }^{\frac{\ell}{2}}$, $\parenth{2\step^{4} \Lparam \Clh}^{\frac{\ell}{2}}$, $\parenth{2\step^{4} \Lparam \Clh}^{\frac{\ell}{2}}$ respectively. Hence,
\begin{align*}
  &\quad \matsnorm{A_3}{\frac{\ell}{2}} \\
  &\leq  \step^{5} \Lparam^{\frac{3}{2}} \Clh  + \step^{5} \Lparam^{\frac{3}{2}} \Clh   +   2\step^{4} \Lparam \Clh  +  2\step^{4} \Lparam \Clh \\
  &\overset{(i)}{\leq} 6 \step^{4} \Lparam \Clh.
\end{align*}
(i) uses the condition that $\step^2 \Lparam \leq \frac{1}{2}$.

Combining the bounds on $A_1, A_2$ and $A_3$ into Eq.~\eqref{eq:B_ell_power} and apply H\"older's inequality, we obtain
\begin{align*}
  \Exs B_\step^{\ell} &\leq \parenth{\Exs B_\step^{\ell}}^{\frac{\ell-1}{\ell}} \parenth{\Exs A_1^{\ell}}^{\frac{1}{\ell}} + \parenth{\Exs B_\step^{\ell}}^{\frac{\ell-2}{\ell}} \parenth{\Exs A_2^{\frac{\ell}{2}}}^{\frac{2}{\ell}} + + \parenth{\Exs B_\step^{\ell}}^{\frac{\ell-2}{\ell}} \parenth{\Exs A_3^{\frac{\ell}{2}}}^{\frac{2}{\ell}} \\
  &\leq \parenth{\Exs B_\step^{\ell}}^{\frac{\ell-1}{\ell}} \parenth{\step^4 \Lparam \trH} + \parenth{\Exs B_\step^{\ell}}^{\frac{\ell-2}{\ell}} \parenth{11 \step^{4} \Lparam \Clh}.
\end{align*}
Solving the inequality equation for $\Exs B_\step^{\ell}$, we obtain
\begin{align*}
  \brackets{\Exs B_\step^{\ell}}^{\frac{1}{\ell}} \leq \step^4 \Lparam \trH + \parenth{11 \step^{4} \Lparam \Clh }^{\frac{1}{2}}.
\end{align*}
Combining the above bound into Eq.~\eqref{eq:Delta_bound_init}, we obtain
\begin{align*}
  \brackets{\Exs_{(q_0, p_0) \sim \mu \times \Normal(0, \Ind_\dims)} \Delta_{\step}^{\ell}}^{\frac{1}{\ell}} \leq 3 \step^4 \Lparam \Cl + 4 \parenth{ \step^{4} \Lparam \Clh }^{\frac{1}{2}}.
\end{align*}
By Markov's inequality, we obtain
\begin{align*}
  \Prob_{(q_0, p_0) \sim \mu \times \Normal(0, \Ind_\dims)}\parenth{\abss{\Delta_\step(q_0, p_0)} \geq \alpha} \leq \frac{\Exs \Delta_\step^{\ell}}{\alpha^{\ell}}.
\end{align*}
Take $\alpha = e^{\frac12} \brackets{3 \step^4 \Lparam \Cl + 4 \parenth{ \step^{4} \Lparam \Clh }^{\frac{1}{2}}}$ and $\ell = 2\ceils{\log\parenth{\frac{1}{\delta}}}$, then the right-hand side becomes less than $e^{-\ell/2} \leq \delta$. Finally, if we take the step-size such that $\step^4 \Lparam \trH \leq \frac{1}{4096}$ and $\step^4 \Lparam^2 \log\parenth{\frac{1}{\delta}} \leq \frac{1}{4096}$, then 
\begin{align*}
  \alpha = e^{\frac12} \brackets{3 \step^4 \Lparam \Cl + 4 \parenth{ \step^{4} \Lparam \Clh}^{\frac{1}{2}}} \leq \frac14. 
\end{align*}

\end{proof}

\section{Proof of results in Section~\ref{sub:intermediate_concentration_results}}
In this section, we prove the intermediate concentration results in Section~\ref{sub:intermediate_concentration_results}. 

\begin{proof}[Proof of Lemma~\ref{lem:grad_norm_bound}.]
  We have
\begin{align*}
  &\quad \Exs_{q \sim e^{-f}} \enorm{\gradf(q)}^{2\ell} \\
  &=  \Exs \vecnorm{\gradf(q)}{2}^{2(\ell-1)} \gradf(q)\tp \gradf(q) \\
  &\overset{(i)}{=} \Exs \enorm{\gradf(q)}^{2(\ell-1)} \trace\parenth{\hessf_q} + (\ell-1) \Exs \enorm{\gradf(q)}^{2(\ell-2)} \trace\parenth{2 \gradf(q) \tp \hessf_q \gradf(q)} \\
  &\leq \Exs \enorm{\gradf(q)}^{2(\ell-1)} \parenth{\trace\parenth{\hessf_q} + 2(\ell-1) \Lparam}.
\end{align*}
(i) applies the multivariate integration by parts (or Green's identities, $\nabla u = \gradf(q) e^{-f(q)}, \nabla v = \vecnorm{\gradf(q)}{2}^{2(\ell-1)} \gradf(q) $), and the boundary term vanishes because $e^{-f}$ is subexponential-tailed. 
By H\"older's inequality, we obtain that
\begin{align*}
  \Exs_{q \sim e^{-f}} \enorm{\gradf(q)}^{2\ell} \leq (\trH + 2 (\ell-1) \Lparam)^\ell = \Cl^\ell.
\end{align*}
\end{proof}

\begin{proof}[Proof of Lemma~\ref{lem:pHp_bound}]
  We have
  \begin{align*}
    &\quad \Exs_{p \sim \Normal(0, \Ind_\dims)} \parenth{p\tp \hessf_{x} p}^{\ell} \\
    &= \Exs_{p \sim \Normal(0, \Ind_\dims)} \parenth{p \tp \hessf_{x} p}^{\ell-1} p \tp \hessf_{x} p \\
    &\overset{(i)}{=} \Exs_{p \sim \Normal(0, \Ind_\dims)} \parenth{p \tp  \hessf_{x} p}^{\ell-1}  \trace\parenth{\hessf_{x}} + 2(\ell-1) \Exs_{p \sim \Normal(0, \Ind_\dims)} \parenth{p \tp \hessf_{x} p}^{\ell-2} p \tp  \hessf_x^2 p \\
    &\leq \Exs_{p \sim \Normal(0, \Ind_\dims)}  \parenth{p \tp  \hessf_{x} p}^{\ell-1} \parenth{\trH + 2 (\ell-1) \Lparam}. 
  \end{align*}
  (i) applies the multivariate integration by parts ($\nabla u = p e^{-\frac{1}{2} \enorm{p}^2}, \nabla v = \parenth{p \tp \hessf_{x} p}^{\ell-1} \hessf_{x} p $) and the boundary term vanishes. By H\"older's inequality, we obtain the desired result. 
\end{proof}

\begin{proof}[Proof of Lemma~\ref{lem:pHp_at_qt_bound}]
The case $t=0$ follows directly from Lemma~\ref{lem:pHp_bound}. In what follows, we assume $t > 0$.
Let $(q_0, p_0)$ be a random variable drawn from $\mu$. Define
\begin{align*}
  \hq_t &\defn q_t = q_0 + tp_0 - \frac{t^2}{2} \gradf(q_0) \\
  \hp_t &\defn p_0 - \frac{t}{2} \gradf(q_0).
\end{align*}
For readers know HMC well, the above corresponds to a half step of the leapfrog integration. 

Define the discrete forward map $\hF_t$ the map that goes from $(q_0, p_0)$ to $(\hq_t, \hp_t)$
\begin{align*}
  \bmat{\hq_t \\ \hp_t} = \hF_t \parenth{\bmat{q_0 \\ p_0}}.
\end{align*} 
Its derivatives are
\begin{align*}
  \bmat{\frac{\partial \hq_t}{\partial q_0} & \frac{\partial \hq_t}{\partial p_0} \\ \frac{\partial \hp_t}{\partial q_0} & \frac{\partial \hp_t}{\partial p_0}} = \bmat{ \Ind_\dims - \frac{t^2}{2} \hessf_{q_0} & t\Ind_\dims \\ -\frac{t}{2}\hessf_{q_0} & \Ind_\dims}.
\end{align*}
The Jacobian $\J \hF_t = 1$. By the implicit function theorem, since the Jacobian vanishes nowhere, $\hF_t$ is invertible. We also have the derivatives of its inverse
\begin{align*}
  \bmat{\frac{\partial q_0}{\partial \hq_t} & \frac{\partial q_0}{\partial \hp_t} \\ \frac{\partial p_0}{\partial \hq_t} & \frac{\partial p_0}{\partial \hp_t}} = \bmat{ \Ind_\dims & - t\Ind_\dims \\ \frac{t}{2} \hessf_{q_0} & \Ind_\dims - \frac{t^2}{2} \hessf_{q_0} }.
\end{align*}

According to the above derivatives, we have
\begin{align*}
  \frac{\partial }{\partial \hp_t} e^{-\frac12 \enorm{p_0}^2} &= -\frac{\partial p_0}{\partial \hp_t} p_0 e^{-\frac12 \enorm{p_0}^2} = -\parenth{\Ind_\dims - \frac{t^2}{2} \hessf_{q_0}} p_0  e^{-\frac12 \enorm{p_0}^2} \\
  \frac{\partial }{\partial \hp_t} e^{-f(q_0)} &= -\frac{\partial q_0}{\partial \hp_t} \gradf(q_0)e^{-f(q_0)} = t \gradf(q_0) e^{-f(q_0)}.
\end{align*}
Take the $\ell$-th moment and write it in a way such that it is ready for integration by parts
\begin{align*}
  &\quad \Exs \parenth{p_0 \tp \hessf_{\hq_t} p_0}^\ell \\
  &= \Exs \parenth{p_0 \tp \hessf_{\hq_t} p_0}^{\ell-1} p_0 \tp \hessf_{\hq_t} \parenth{\Ind_\dims - \frac{t^2}{2} \hessf_{q_0}} p_0 + \Exs \parenth{p_0 \tp \hessf_{\hq_t} p_0}^{\ell-1} p_0 \tp \hessf_{\hq_t} \frac{t^2}{2} \hessf_{q_0} p_0.
\end{align*}
The first term above is
\begin{align*}
  &\quad \Exs \parenth{p_0 \tp \hessf_{\hq_t} p_0}^{\ell-1} p_0 \tp \hessf_{\hq_t} \parenth{\Ind_\dims - \frac{t^2}{2} \hessf_{q_0}} p_0 \\
  &= \int \int \parenth{p_0 \tp \hessf_{\hq_t} p_0}^{\ell-1} p_0 \tp \hessf_{\hq_t} \parenth{\Ind_\dims - \frac{t^2}{2} \hessf_{q_0}} p_0 e^{-f(q_0)} e^{-\frac12 \enorm{p_0}^2} \frac{1}{(2\pi)^{\dims/2}} dp_0 dq_0 \\ 
  &\overset{(i)}{=} \int \int \parenth{p_0 \tp \hessf_{\hq_t} p_0}^{\ell-1} p_0 \tp \hessf_{\hq_t} \parenth{\Ind_\dims - \frac{t^2}{2} \hessf_{q_0}} p_0 e^{-f(q_0)} e^{-\frac12 \enorm{p_0}^2} \frac{1}{(2\pi)^{\dims/2}} d\hp_t d\hq_t \\ 
  &= -\int \int \parenth{p_0 \tp \hessf_{\hq_t} p_0}^{\ell-1} e^{-f(q_0)} p_0 \tp \hessf_{\hq_t} \frac{\partial }{\partial \hp_t}  e^{-\frac12 \enorm{p_0}^2}  \frac{1}{(2\pi)^{\dims/2}} d\hp_t d\hq_t \\ 
  &\overset{(ii)}{=} \Exs \parenth{p_0 \tp \hessf_{\hq_t} p_0}^{\ell-1} \trace\brackets{\parenth{\Ind_\dims - \frac{t^2}{2} \hessf_{q_0}} \hessf_{\hq_t} + t p_0 \tp \hessf_{\hq_t} \gradf(q_0) } \\
  & \quad + 2(\ell-1) \Exs \parenth{p_0 \tp \hessf_{\hq_t} p_0}^{\ell-2} p_0 \tp \hessf_{\hq_t} (\Ind_\dims - \frac{t^2}{2} \hessf_{q_0}) \hessf_{\hq_t} p_0 \\
  &\leq \Exs \parenth{p_0 \tp \hessf_{\hq_t} p_0}^{\ell-\frac12} t \Lparam^{\frac12}\enorm{\gradf(q_0)} + \Exs \parenth{p_0 \tp \hessf_{\hq_t} p_0}^{\ell-1} \parenth{\trH + 2 (\ell-1) \Lparam}.
\end{align*}
(i) follows from change of variable and the Jacobian is $1$. (ii) applies multivariate integration by parts ($\nabla u = \frac{\partial }{\partial \hp_t}  e^{-\frac12 \enorm{p_0}^2}$ with respect to $\hp_t$, $\nabla v$ is the rest). The boundary term vanishes as $e^{-f}$ is subexponential-tailed.  

Combined with the second term, we obtain 
\begin{align*}
  &\quad \Exs \parenth{p_0 \tp \hessf_{\hq_t} p_0}^\ell \\
  &\leq  \Exs \parenth{p_0 \tp \hessf_{\hq_t} p_0}^{\ell-\frac12} \parenth{t \Lparam^{\frac12}\enorm{\gradf(q_0)} + t^2 \Lparam \enorm{\hessf_{q_0}^{\frac12} p_0}} + \Exs \parenth{p_0 \tp \hessf_{\hq_t} p_0}^{\ell-1} \parenth{\trH + 2 (\ell-1) \Lparam}.
\end{align*}
By H\"older's inequality, we have
\begin{align*}
  &\quad \Exs \parenth{p_0 \tp \hessf_{\hq_t} p_0}^\ell \\
  &\leq \parenth{t \Lparam^{\frac12}}^{2\ell} \Exs  \enorm{\gradf(q_0)}^{2\ell} +  \parenth{t^2 \Lparam}^{2\ell} \Exs  \parenth{p_0 \hessf_{q_0} p_0}^{\ell} + \parenth{\trH + 2 (\ell-1) \Lparam}^{\ell}.
\end{align*}
Together with Lemma~\ref{lem:grad_norm_bound} and Lemma~\ref{lem:pHp_bound}, we conclude that
\begin{align*}
  \brackets{\Exs \parenth{p_0 \tp \hessf_{\hq_t} p_0}^\ell}^{\frac{1}{\ell}} \leq (1+ t\Lparam^{\frac12}) \parenth{\trH + 2 (\ell-1) \Lparam} \leq 2 Cl.
\end{align*}
\end{proof}

\begin{proof}[Proof of Lemma~\ref{lem:grad_diff_norm_bound}]
  Since we have
\begin{align*}
  \enorm{\gradf(q_t) - \gradf(q_0)} &= \enorm{\int_0^t \hessf_{q_s} \parenth{p_0 - s \gradf(q_0)} ds } \\
  &\leq \int_0^t \Lparam^{\frac12} \parenth{p_0 \tp \hessf_{q_s} p_0}^{\frac12} + s \Lparam \enorm{\gradf(q_0)} ds,
\end{align*}
it follows that 
\begin{align*}
  \Exs_{(q_0, p_0) \sim \mu} \enorm{\gradf(q_t) - \gradf(q_0)}^{2\ell} &\leq \Exs t^{2\ell-1} \int_0^t \parenth{\Lparam^{\frac12} \parenth{p_0 \tp \hessf_{q_s} p_0}^{\frac12} + s \Lparam \enorm{\gradf(q_0)}}^{2\ell} ds \\
  &\leq \Exs t^{2\ell-1} 2^{\ell-1} \int_0^t \parenth{\Lparam^\ell \parenth{p_0 \tp \hessf_{q_s} p_0}^{\ell} + s^{2\ell} \Lparam^{2\ell} \enorm{\gradf(q_0)}^{2\ell}} ds \\
  &\overset{(i)}{\leq} \parenth{4 t^2 \Lparam \Cl}^{\ell}.
\end{align*}
Here (i) follows from Lemma~\ref{lem:grad_norm_bound} and~\ref{lem:pHp_at_qt_bound}.
\end{proof}


\bibliographystyle{alpha}
\bibliography{ref}

\appendix

\end{document}